\documentclass{article}

\usepackage[utf8]{inputenc}
\usepackage[T1]{fontenc}
\usepackage{hyperref}
\usepackage{url}
\usepackage{booktabs}
\usepackage{amsfonts}
\usepackage{nicefrac}
\usepackage{microtype}
\usepackage{graphicx}
\usepackage{amsmath,amssymb,amsthm}
\usepackage{bm}
\usepackage{geometry}
\usepackage{natbib}
\usepackage{algorithm}
\usepackage{algorithmic}
\usepackage{xcolor}
\usepackage{enumitem}

\newtheorem{proposition}{Proposition}

\newtheorem{lemma}{Lemma}

\title{Knob: A Physics-Inspired Gating Interface for \\ Interpretable and Controllable Neural Dynamics}

\author{
  Siyu Jiang\textsuperscript{1}, Sanshuai Cui\textsuperscript{1*}, Hui Zeng\textsuperscript{2}\\
  \textsuperscript{1}City University of Macau, Macau, China\\
  \textsuperscript{2}Southwest University of Science and Technology, Mianyang, China\\
  \texttt{\{D23090100052; sanshuaicui\}@cityu.edu.mo; zengh5@mail2.sysu.edu.cn}\\
  \textsuperscript{*} Corresponding author.
}

\date{}

\begin{document}
\maketitle

\begin{abstract}
Existing neural network calibration methods often treat calibration as a static, post-hoc optimization task. However, this neglects the dynamic and temporal nature of real-world inference. Moreover, existing methods do not provide an intuitive interface enabling human operators to dynamically adjust model behavior under shifting conditions. In this work, we propose \textbf{Knob}, a framework that connects deep learning with classical control theory by mapping neural gating dynamics to a second-order mechanical system. By establishing correspondences between physical parameters---damping ratio ($\zeta$) and natural frequency ($\omega_n$)---and neural gating, we create a tunable ``safety valve.'' The core mechanism employs a logit-level convex fusion, functioning as an input-adaptive temperature scaling. It tends to reduce model confidence particularly when model branches produce conflicting predictions. Furthermore, by imposing second-order dynamics (Knob-ODE), we enable a \textit{dual-mode} inference: standard i.i.d. processing for static tasks, and state-preserving processing for continuous streams. Our framework allows operators to tune ``stability'' and ``sensitivity'' through familiar physical analogues. \textbf{This paper presents an exploratory architectural interface}; we focus on demonstrating the concept and validating its control-theoretic properties rather than claiming state-of-the-art calibration performance. Experiments on CIFAR-10-C validate the calibration mechanism and demonstrate that, in Continuous Mode, the gate responses are consistent with standard second-order control signatures (step settling and low-pass attenuation), paving the way for predictable human-in-the-loop tuning.
\end{abstract}

\section{Introduction}
\label{sec:intro}

Deep neural networks have achieved remarkable success across various domains. However, their ``black box'' nature remains problematic, especially regarding their ability to reliably estimate uncertainty. In safety-critical applications like autonomous driving or medical diagnosis, a model's ability to express uncertainty---and correct for overconfidence---is as crucial as its accuracy. The well-documented robustness--calibration paradox \citep{guo2017calibration,ovadia2019uncertainty}, where models become less calibrated as they become more robust, is a critical challenge \citep{Roschewitz2025ShiftCalSurvey}. Models that cannot faithfully represent their uncertainty under novel conditions are inherently unreliable.

Existing solutions to this problem typically fall into two categories: post-hoc methods and training-time regularization. Post-hoc techniques like Temperature Scaling (TS) \citep{guo2017calibration} adjust confidence scores using a validation set, effectively fitting a single scalar to a dataset that may not resemble the test environment. Training-time methods, such as label smoothing \citep{liu2022mbls} or Mixup \citep{Noh2023RankMixup}, attempt to bake calibration into the weights. However, both approaches suffer from a fundamental limitation: \textbf{opacity}. They produce static solutions that are difficult for an end-user to interpret or adjust. If a deployed model behaves erratically, the operator has no intuitive lever to pull; the only recourse is often to retrain or recalibrate on new data, which is not always feasible.

We propose a different paradigm: viewing neural network inference not as a static mapping, but as a \textbf{dynamic system with a controllable ``Volume Knob.''} Just as an audio engineer adjusts compressors to manage volume spikes or a driver tunes suspension settings for different terrains, an AI operator should have intuitive means to control a model's ``conservativeness'' or ``responsiveness.'' We argue that the language of classical mechanics---damping, stiffness, and inertia---provides the ideal vocabulary for this interface.

A more promising direction lies in architectural design. The superior calibration of architectures like Vision Transformers (ViT) \citep{minderer2021revisiting} suggests that calibration can be an intrinsic property of the model itself. However, such properties often emerge without a clear theoretical framework or interpretable control mechanisms. Drawing inspiration from the principled stability criteria in classical control theory \citep{franklin2019feedback}, we propose embedding the dynamics of physical systems directly into neural network architectures.

Our primary contribution is \textbf{Knob}, a physics-grounded architectural framework that establishes a formal, differentiable mapping between the parameters of a classical second-order damped mechanical system (e.g., its damping ratio $\zeta$ and natural frequency $\omega_n$) and the inference-time dynamics of a neural network. This mapping is enabled by a novel, physics-inspired neural layer that governs a gate's response with interpretable control over its stability and speed, while Tustin discretization ensures step-size-independent stability. This gate performs a logit-level convex fusion, which we interpret as an input-adaptive temperature scaling mechanism that tends to suppress overconfidence. We instantiate this framework in a family of efficient methods (Knob) and illustrate our claims with a cohesive theory-evidence loop: the confidence moderation properties of convex fusion are examined in E-1, the gate's learned behavior in E-2, and the second-order dynamics in E-3.

\textbf{Scope of this work.} This paper is exploratory in nature: we focus on introducing the \textit{architectural interface} and demonstrating that the control-theoretic properties (damping, frequency response) are indeed functional. We do not claim state-of-the-art calibration; rather, we aim to open a new design axis where model behavior can be adjusted via physically meaningful parameters.

\section{Related Work}
\label{sec:related}

The challenge of ensuring that neural network confidence scores accurately reflect true prediction correctness, particularly under distribution shift, has spurred extensive research. The field has predominantly advanced along two main trajectories: post-hoc calibration and training-time regularization.

\paragraph{Post-hoc Calibration.}
Post-hoc methods, such as the seminal Temperature Scaling (TS) \citep{guo2017calibration}, its input-adaptive extensions \citep{mozafari2018ats,joy2023sats}, and more complex density-based approaches like Dirichlet calibration \citep{kull2019dirichlet} or Bayesian Binning into Quantiles (BBQ) \citep{naeinibbq}, adjust the outputs of a pre-trained model. Although computationally inexpensive, their dependency on representative validation sets renders them vulnerable to performance deterioration under non-stationary real-world data streams \citep{ovadia2019uncertainty}.

\paragraph{Training-time Regularization.}
Training-time regularization techniques, including label smoothing \citep{liu2022mbls} and data augmentation strategies like Mixup \citep{Noh2023RankMixup}, aim to build in calibration from the outset by discouraging overconfident predictions during optimization. These methods can improve in-distribution calibration but often require careful hyperparameter tuning and can increase training costs, limiting their agility.

\paragraph{Architectural Design for Calibration.}
Our work carves a third path: architectural design for calibration. While some studies have shown that certain architectures like Vision Transformers (ViT) inherently offer better calibration \citep{minderer2021revisiting}, many such findings are empirical observations rather than the result of principled design. We distinguish our approach by explicitly importing concepts from classical control theory \citep{franklin2019feedback}. Instead of relying on emergent properties, we engineer the network's gating dynamics to follow a prescribed second-order damped system, providing an interpretable, physics-grounded mechanism for confidence modulation that is intrinsic to the model's structure.

\paragraph{Physics-Informed Deep Learning.}
Integrating physical priors into deep learning typically involves solving partial differential equations (PDEs) or modeling physical systems. Neural ODEs \citep{chen2018neuralode} bridged the gap between residual networks and dynamical systems. However, most prior work uses AI to \textit{solve} physics. In contrast, we use physics to \textit{constrain} AI. By imposing second-order damping dynamics on the gating mechanism, we leverage the stability guarantees of classical mechanics to regulate the volatile confidence estimates of deep networks.

\paragraph{Second-order Neural ODEs and Momentum Dynamics.}
Several works explore second-order dynamics in neural ODEs (e.g., heavy-ball style formulations) primarily to improve training efficiency or feature evolution. Knob differs in intent: we use a second-order prior to regulate \emph{inference-time gating} and expose $(\zeta,\omega_n)$ as user-interpretable control surfaces for run-time behavior shaping, rather than optimizing the training trajectory.

\paragraph{Oscillatory and Damping-based Priors.}
Damped harmonic oscillator models are widely used in control theory to describe stable responses under noise and perturbations. In our setting, we adopt the same control primitives (damping ratio and bandwidth) but apply them to a neural gate, allowing us to probe step and frequency responses (E-3) as diagnostic signatures of the embedded prior.

\paragraph{Comparison with MoE and Ensembling.}
Our approach shares similarities with Mixture of Experts (MoE) \citep{shazeer2017moe} and Deep Ensembles \citep{lakshminarayanan2017simple}. However, standard MoE gates are typically trained to maximize capacity or sparsity and are theoretically unconstrained. Deep Ensembles rely on averaging independent models, which is computationally expensive and lacks run-time adjustability. Knob differs by explicitly constraining the gating signal to follow a \textit{differential equation}, ensuring that the mixing weights evolve smoothly and predictably according to user-intelligible parameters ($\zeta, \omega_n$).

\section{Method: The Knob Framework}
\label{sec:method}

This section details the Knob framework, proceeding from its high-level structure to its theoretical underpinnings and dynamic properties.

\subsection{Architectural Overview and Core Terminology}
\label{sec:method-overview}

The core of the Knob framework is a \textbf{logit-level convex fusion} mechanism that operates on two parallel logit streams, $\bm z_{\mathrm{static}}$ and $\bm z_{\mathrm{dyn}}$, supplied by a dual-stream backbone (Fig.~\ref{fig:framework}). A sample-wise, learned scalar gate $g(x) \in [0,1]$ computes a convex combination of these two streams:
\begin{equation}
\label{eq:method-convex}
\bm z_{\mathrm{fuse}} \;=\; g(x)\,\bm z_{\mathrm{dyn}} \;+\; \bigl(1-g(x)\bigr)\,\bm z_{\mathrm{static}}.
\end{equation}
This architecture ensures that the fused logits $\bm z_{\mathrm{fuse}}$ lie on the line segment connecting the two source logits, a property that intrinsically limits overconfidence, as we will formalize in \S\ref{sec:method-properties}. To provide stable and interpretable control over the gate's behavior, its dynamics are governed by a differentiable second-order system, detailed in \S\ref{sec:method-dynamics}.

\begin{figure}[t]
\centering
  \includegraphics[width=0.85\linewidth]{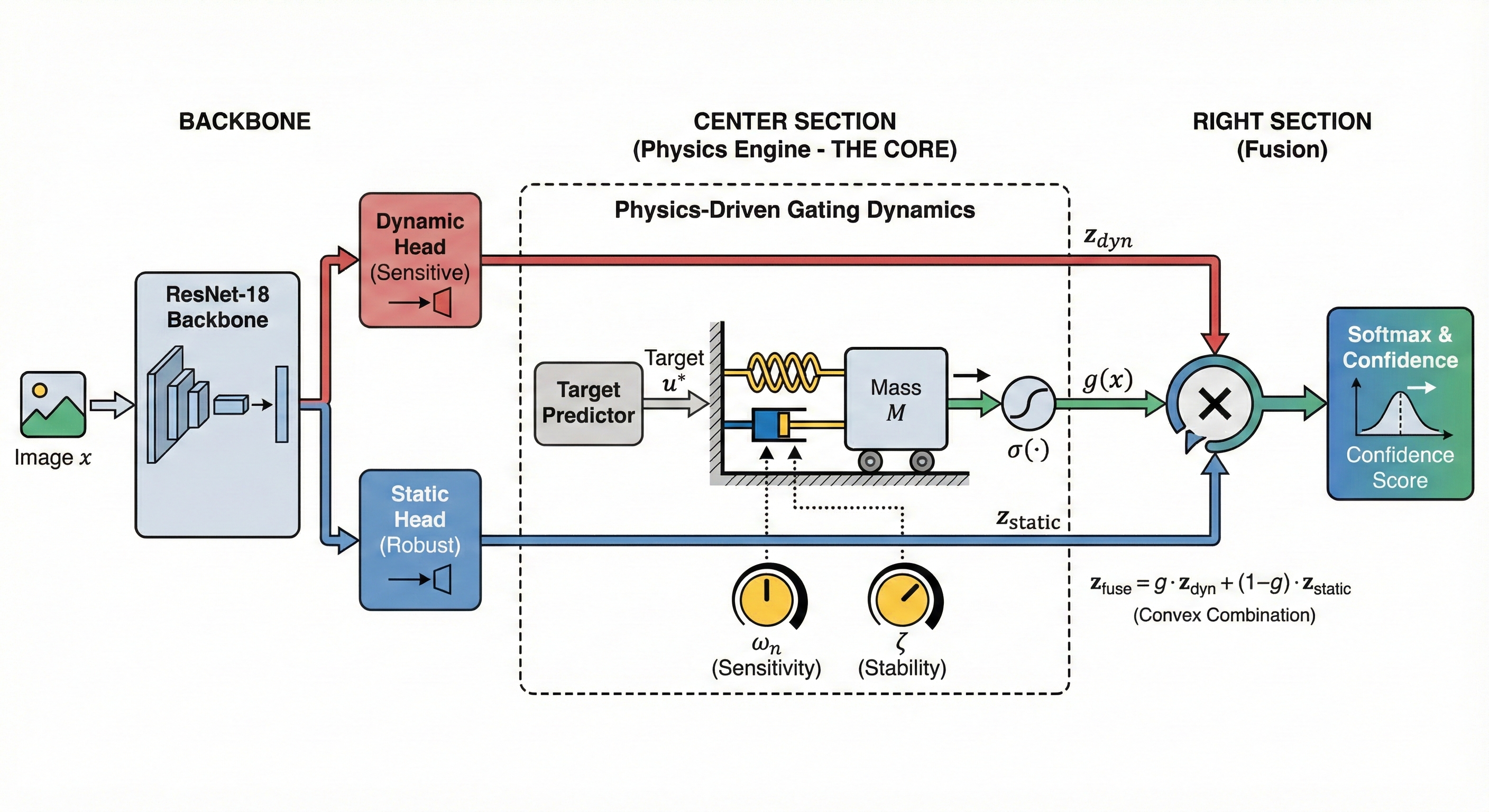}
  \caption{\textbf{The Knob framework as a physics-inspired interface.} 
  \textit{Left (Dual-Stream Backbone):} A shared encoder feeds two lightweight projection heads, producing a robust ``Static'' branch and a more sensitive ``Dynamic'' branch, yielding complementary logit vectors $\bm{z}_{\mathrm{static}}$ and $\bm{z}_{\mathrm{dyn}}$.
  \textit{Center (Physics Engine):} The gating mechanism is modeled as a \textbf{mass-spring-damper} system (normalized to unit mass). A network-predicted reference input $u^*(x)$ drives the system; its response is governed by two interpretable control parameters---\textbf{Natural Frequency} ($\omega_n$), controlling sensitivity/bandwidth, and \textbf{Damping Ratio} ($\zeta$), controlling stability/conservativeness---which act as tunable ``knobs'' for the operator.
  \textit{Right (Convex Fusion):} The gate value $g(x) \in [0,1]$, determined by the mass displacement, performs a convex combination of two logit branches, naturally limiting model overconfidence.}
  \label{fig:framework}
\end{figure}

For clarity, we systematically define the family of methods evaluated in this paper:
\begin{itemize}
    \item \textbf{Static-only (Static):} Uses only the static branch logits as baseline.
    \item \textbf{Channel-attention fusion (Attention):} A non-convex counterpart using standard attention.
    \item \textbf{Input-adaptive convex gate (Knob-IA):} The core proposal using the learned gate $g(x)$ for fusion.
    \item \textbf{EMA-smoothed gate (ODE-Lite):} A lightweight version using first-order Exponential Moving Average.
    \item \textbf{Second-order damped gate (Knob-ODE):} The full physics-inspired dynamics.
\end{itemize}

\subsection{The Second-Order Damped Gate: Physical Control Parameters}
\label{sec:method-dynamics}

To prevent the learned gate $g(x)$ from introducing instability, we constrain its dynamics by modeling its evolution as a classical second-order damped system. This approach injects a strong, interpretable physical prior into the model's inference-time behavior, drawing inspiration from neural differential equations \citep{chen2018neuralode}.

\paragraph{From physical equations to a state-space representation.}
Instead of directly modeling the gate value $g \in [0,1]$, we impose dynamics on a latent variable $u$ and obtain the gate via a sigmoid mapping, $g=\sigma(u)$. The evolution of $u$ follows a standard second-order ODE:
\begin{equation}
\ddot{u} + 2\zeta\omega_n \dot{u} + \omega_n^2 u \;=\; \omega_n^2\,u^*(x),
\end{equation}
where $u^*(x)$ is a network-predicted ``target displacement,'' and $\zeta$ (damping ratio) and $\omega_n$ (natural frequency) are learnable physical parameters governing the system's response. 

\paragraph{Decoupling content from dynamics.}
A key benefit of Eq.~(2) is that it separates \emph{what} the gate wants to do from \emph{how} it gets there.
The reference input $u^*(x)$ (learned from data) specifies the desired operating point for the gate,
while the physical parameters $(\zeta,\omega_n)$ explicitly shape the transition trajectory
(e.g., overshoot, settling speed, and bandwidth).
This makes the inference-time behavior tunable through interpretable control knobs, unlike standard
gating mechanisms whose dynamics are implicitly entangled in weight matrices.

The key innovation is the re-interpretation of these coefficients as \textbf{user-facing control parameters}:
\begin{itemize}
    \item \textbf{$\zeta$ (Damping Ratio) $\to$ Conservativeness/Stability:} This parameter controls the system's resistance to change. A high $\zeta$ (overdamped) corresponds to a conservative policy where the gate changes position slowly and deliberately, avoiding ``knee-jerk'' reactions to noisy inputs. A low $\zeta$ (underdamped) allows for faster reaction but risks ``ringing'' or instability.
    \item \textbf{$\omega_n$ (Natural Frequency) $\to$ Sensitivity/Bandwidth:} This controls how fast the system \textit{can} respond. A low $\omega_n$ acts as a low-pass filter, ignoring high-frequency input noise (or adversarial perturbations) and focusing on the steady-state signal. A high $\omega_n$ makes the system highly sensitive to every input fluctuation.
\end{itemize}

For efficient neural implementation, we convert this to a first-order state-space representation with state vector $\bm{x}=[u,\dot{u}]^{\top}$:
\begin{equation}
\dot{\bm{x}} \;=\; A\,\bm{x} \;+\; B\,u^*(x), 
\qquad 
A \;=\; \begin{bmatrix} 0 & 1 \\ -\omega_n^2 & -2\zeta\omega_n \end{bmatrix}, 
\qquad 
B \;=\; \begin{bmatrix} 0 \\ \omega_n^2 \end{bmatrix}.
\end{equation}
The matrix $A$ governs the system's internal dynamics, while $B$ defines how the input $u^*(x)$ influences the state.

\paragraph{Tustin discretization and the forward pass.}
For discrete-time neural computation, we use the Tustin (bilinear) transform \citep{smith1997dspguide} to discretize the continuous system. This method preserves stability and maps the continuous matrices $(A,B)$ to their discrete counterparts $(A_d,B_d)$:
\begin{equation}
A_d \;=\; \bigl(I - \tfrac{\Delta t}{2}A\bigr)^{-1}\!\bigl(I + \tfrac{\Delta t}{2}A\bigr),
\qquad
B_d \;=\; \bigl(I - \tfrac{\Delta t}{2}A\bigr)^{-1} \Delta t\, B.
\end{equation}

\begin{proposition}[Step-size--independent stability of Tustin discretization]
\label{prop:tustin-stable}
If the continuous-time system parameters satisfy $\zeta>0$ and $\omega_n>0$, then the eigenvalues of the discretized matrix $A_d$ obtained via the Tustin transform obey $|\lambda_i|<1$ for all $\Delta t>0$. Hence the resulting discrete-time system is stable independently of the step size (see Appendix~\ref{app:tustin-proof} for proof).
\end{proposition}

The forward pass for \textbf{Knob-ODE} is governed by the linear recursion: $\bm{x}_t = A_d\bm{x}_{t-1} + B_d u_t^*$. The final gate value is $g_t=\sigma\!\bigl(\bm{x}_t[0]\bigr)$.

\paragraph{Dual-Mode Inference: Static vs. Continuous.}
Crucially, our framework supports two inference modes depending on the application context:
\begin{enumerate}[leftmargin=*]
    \item \textbf{Reset Mode (I.I.D. Tasks):} For standard benchmarks like ImageNet or CIFAR, where samples are independent, we reset the internal state $\bm{x}$ to zero for each new input ($\bm{x}_0 = \bm{0}$). In this mode, the mechanism acts as a single-step, input-adaptive gate without temporal memory.
    \item \textbf{Continuous Mode (Stream Processing):} For time-series data, video streams, or our dynamic probe experiments (E-3), the state $\bm{x}_t$ is preserved across time steps. This allows the ``physical inertia'' of the gate to filter out high-frequency noise and enforce temporal consistency, strictly adhering to the specified damping dynamics.
\end{enumerate}
Unless otherwise stated, quantitative metrics (Table~\ref{tab:main}) use Reset Mode, while dynamic analyses (Fig.~\ref{fig:e3-combined}) use Continuous Mode. In our implementation, we use a default time step of $\Delta t=1$.

\paragraph{What ODE priors do in Reset Mode.}
In Reset Mode, the state is initialized as $\bm{x}_0=\bm{0}$ and we apply a single update per sample.
Even without temporal memory, the discretized dynamics still induce a parametric shrinkage of the raw command $u^*(x)$ (and thus of $g$), acting as a lightweight, physically-motivated regularizer.
In Continuous Mode, preserving $\bm{x}_t$ additionally yields genuine temporal smoothing.

\paragraph{Contrast to standard neural gates.}
Unlike common gates (e.g., sigmoid gates in recurrent units) whose dynamics are implicit in learned weights, Knob parameterizes gate dynamics explicitly with a second-order template and stabilizes discretization via Tustin. This yields controllable response regimes (over/critical/under-damped) and interpretable bandwidth through $(\zeta,\omega_n)$.

\subsection{Interpretation of Convex Fusion}
\label{sec:method-properties}

The logit-level convex fusion (Eq.~\ref{eq:method-convex}) provides a geometric mechanism for confidence moderation. We formalize this behavior under explicit Top-1/Top-2 conditions.

\begin{proposition}[Convex Fusion Contracts the Top-2 Margin]
\label{prop:adaptive-temp}
Let $\bm z_1, \bm z_2 \in \mathbb{R}^K$ and $g\in[0,1]$. Assume the two branches agree on the predicted class
$k^\star=\arg\max_k z_1^{(k)}=\arg\max_k z_2^{(k)}$, and that the runner-up class $j^\star$ (Top-2 index) is preserved.
Let $m_i = z_i^{(k^\star)} - z_i^{(j^\star)}$ be the Top-2 margin of branch $i$. Then the fused Top-2 margin satisfies
\begin{equation}
m_{\mathrm{fuse}} = g\,m_1 + (1-g)\,m_2 \in [\min(m_1,m_2),\,\max(m_1,m_2)],
\end{equation}
and in particular $m_{\mathrm{fuse}} \le \max(m_1,m_2)$ with strict inequality when $g\in(0,1)$ and $m_1\neq m_2$.
\end{proposition}

\textit{Remark.} In a Top-2 reduction, temperature scaling rescales the margin as $m \mapsto m/T$.
Thus one can associate an effective temperature $T_{\mathrm{eff}}:=\max(m_1,m_2)/m_{\mathrm{fuse}} \ge 1$,
which becomes larger when the two branches disagree in their margins, yielding a more conservative distribution.

To quantify confidence moderation, we report the confidence--shrinkage ratio (CSR) under a Top-2 diagnostic view.
Let $s(m)=1/(1+\exp(-m))$ map a logit margin to a scalar confidence proxy (exact for binary, used here as a probe when Top-2 is preserved).
Using the effective temperature view $m_{\mathrm{fuse}}=\max(m_1,m_2)/T_{\mathrm{eff}}$, we define
\begin{equation}
\mathrm{CSR}(x) := \frac{s\!\big(\max(m_1,m_2)/T_{\mathrm{eff}}(x)\big)}{s\!\big(\max(m_1,m_2)\big)}.
\end{equation}
Empirically, $\mathrm{CSR}(x)<1$ is most pronounced on disagreement subsets (Fig.~\ref{fig:e1}), indicating more conservative predictions when branches disagree. For fixed branch logits, the cross-entropy is convex in the scalar gate $g$, which makes optimization of the gate well-behaved in isolation (Appendix~\ref{app:gate-convex}).

\subsection{Complexity and Implementation}
\label{sec:method-complexity}

The Knob family is designed for efficiency, as the gating mechanism introduces only $\mathcal{O}(D)$ overhead. This is because the target displacement network $u^*(x)$ is a lightweight MLP and the physical parameters $\zeta$ and $\omega_n$ are scalars, resulting in negligible increases in parameters, GFLOPs, or latency, as shown in Table~\ref{tab:efficiency}. The lightweight \textbf{ODE-Lite} variant, which uses a simple first-order exponential moving average (EMA) on the gate, offers a particularly attractive trade-off between performance and cost.

\section{Experiments}
\label{sec:exp}

Our experiments aim to: (1) compare the Knob framework's overall performance with established baselines on a benchmark featuring standard distribution shifts, and (2) validate theoretical claims through targeted empirical probe experiments.

\subsection{Experimental Setup and Metrics}
\label{sec:exp-setup}

\textbf{Datasets and Model.} We use \textsc{CIFAR-10}~\citep{krizhevsky2009learning} for training and \textsc{CIFAR-10-C}~\citep{hendrycks2019benchmarking} for evaluation under distribution shift. CIFAR-10-C subjects the test set to 19 different types of corruptions (e.g., Gaussian noise, blur, snow) across 5 severity levels, serving as a rigorous stress test for model calibration. Our backbone model is a \textbf{ResNet-18}~\citep{he2016resnet}. All methods share the same architecture and training protocol for fair comparison.

\textbf{Training Protocol.} Models are trained for 100 epochs using AdamW optimizer with automatic mixed precision. We employ a curriculum sampling strategy that gradually increases corruption severity, which we found stabilizes the learning of the gating mechanism (see Appendix~\ref{app:hyperparams} for hyperparameter details).

\textbf{Evaluation Metrics.} To provide a comprehensive assessment, we use a suite of metrics focused on accuracy, calibration, and efficiency:
\begin{itemize}
    \item \textbf{Avg-C (\%):} Mean accuracy averaged first over corruption severities, then over corruption types.
    \item \textbf{ECE$_{\mathrm{deb}}$:} Debiased Expected Calibration Error; measures gap between confidence and accuracy (lower is better).
    \item \textbf{Err-C (\%):} Mean classification error on CIFAR-10-C, defined as $\mathrm{Err\text{-}C}=100-\mathrm{Avg\text{-}C}$ (lower is better). We report it alongside Avg-C for readability.
    \item \textbf{GFLOPs / Latency (ms):} Computational efficiency metrics.
\end{itemize}

\begin{table}[t]
\centering
\caption{\textbf{Terminology and Metric Definitions.}}
\label{tab:metrics-def}
\small
\begin{tabular}{p{2.8cm}p{9.2cm}}
\toprule
\textbf{Symbol/Term} & \textbf{Definition and Intuition} \\
\midrule
\multicolumn{2}{l}{\textit{\textbf{Methodology Symbols}}} \\
$z_{\text{static}}, z_{\text{dyn}}$ & Logit vectors from the static and dynamic branches, respectively. \\
$g(x)$ & A sample-wise scalar gate value in $[0,1]$ that controls the convex fusion. \\
$z_{\text{fuse}}$ & The final logit vector after convex fusion. \\
$T_{\mathrm{eff}}(x)$ & The effective temperature; a measure of induced smoothing ($T_{\mathrm{eff}} \ge 1$). \\
$\zeta, \omega_n$ & Damping ratio and natural frequency; control overshoot and convergence speed. \\
\midrule
\multicolumn{2}{l}{\textit{\textbf{Evaluation Metrics}}} \\
CSR & Confidence--Shrinkage Ratio; measures how much fusion reduces confidence. \\
Err-C (\%) & Mean classification error on CIFAR-10-C, $\mathrm{Err\text{-}C}=100-\mathrm{Avg\text{-}C}$ (lower is better). \\
NLL & Negative Log-Likelihood; measures predictive uncertainty (lower is better). \\
Brier Score & Mean squared error between predicted probabilities and outcomes. \\
\bottomrule
\end{tabular}
\end{table}

\subsection{Main Results}
\label{sec:exp-main-results}

\paragraph{Conclusion-first summary.}
The main results (Table~\ref{tab:main}, Fig.~\ref{fig:pareto}) suggest the following trade-offs:
\begin{itemize}
    \item \textbf{Efficiency-focused calibration:} \textbf{ODE-Lite} tends to lie near the calibration--cost frontier in our setting.
    \item \textbf{Accuracy:} \textbf{Knob-IA} achieves slightly higher Avg-C in this benchmark, while adding the interpretability of convex fusion.
    \item \textbf{Explicit dynamics:} \textbf{Knob-ODE} is designed to expose stable second-order gate dynamics in \emph{Continuous Mode}; we analyze these control signatures in E-3 rather than claiming universal metric dominance.
\end{itemize}

\begin{table}[t]
\caption{\textbf{Main results on CIFAR-10-C (Reset Mode).} Calibration and efficiency under identical budgets (3-run averages). Lower is better for Err-C and ECE$_{\mathrm{deb}}$. Avg-C = mean accuracy averaged over severities and corruptions.}
\label{tab:main}
\centering
\small
\begin{tabular}{lcccc}
\toprule
Method & Avg-C (\%) $\uparrow$ & Err-C (\%) $\downarrow$ & ECE$_{\mathrm{deb}}$ $\downarrow$ & Latency (ms) \\
\midrule
\textbf{Static}      & 71.25 & 28.75 & 0.1553 & 3.48 \\
\textbf{Attention}   & 70.45 & 29.55 & 0.1648 & 3.51 \\
\textbf{Knob-IA}     & \textbf{72.52} & \textbf{27.48} & 0.1617 & 3.49 \\
\textbf{ODE-Lite}    & 71.62 & 28.38 & \textbf{0.1488} & 3.64 \\
\textbf{Knob-ODE}    & 71.49 & 28.51 & 0.1829 & 4.41 \\
\bottomrule
\end{tabular}
\end{table}

\paragraph{Detailed results.}
Under matched budgets, \textbf{ODE-Lite} attains the best debiased ECE$_{\mathrm{deb}}$ as well as the strongest cost-normalized calibration; \textbf{Knob-IA} stands out on Avg-C (and its complement Err-C); and \textbf{Knob-ODE} delivers the explicit dynamic interface when higher latency budgets are acceptable. The success of the lightweight variant suggests that the core principle---temporal smoothing and convex fusion---is robust even when the physics are approximated.

\textit{Note on Knob-ODE calibration in Reset Mode.} The higher ECE$_{\mathrm{deb}}$ of \textbf{Knob-ODE} in Table~\ref{tab:main} reflects a design trade-off: Knob-ODE is architecturally optimized for \textbf{Continuous Mode}, where its second-order dynamics provide temporal smoothing and stability (validated in E-3). In Reset Mode, the ODE state is reinitialized for each sample, which prevents the physical prior from fully manifesting. Users seeking optimal i.i.d. calibration should prefer \textbf{ODE-Lite} or \textbf{Knob-IA}; Knob-ODE is recommended when stream-level temporal consistency is the priority.

\begin{figure}[t]
  \centering
  \includegraphics[width=0.48\linewidth]{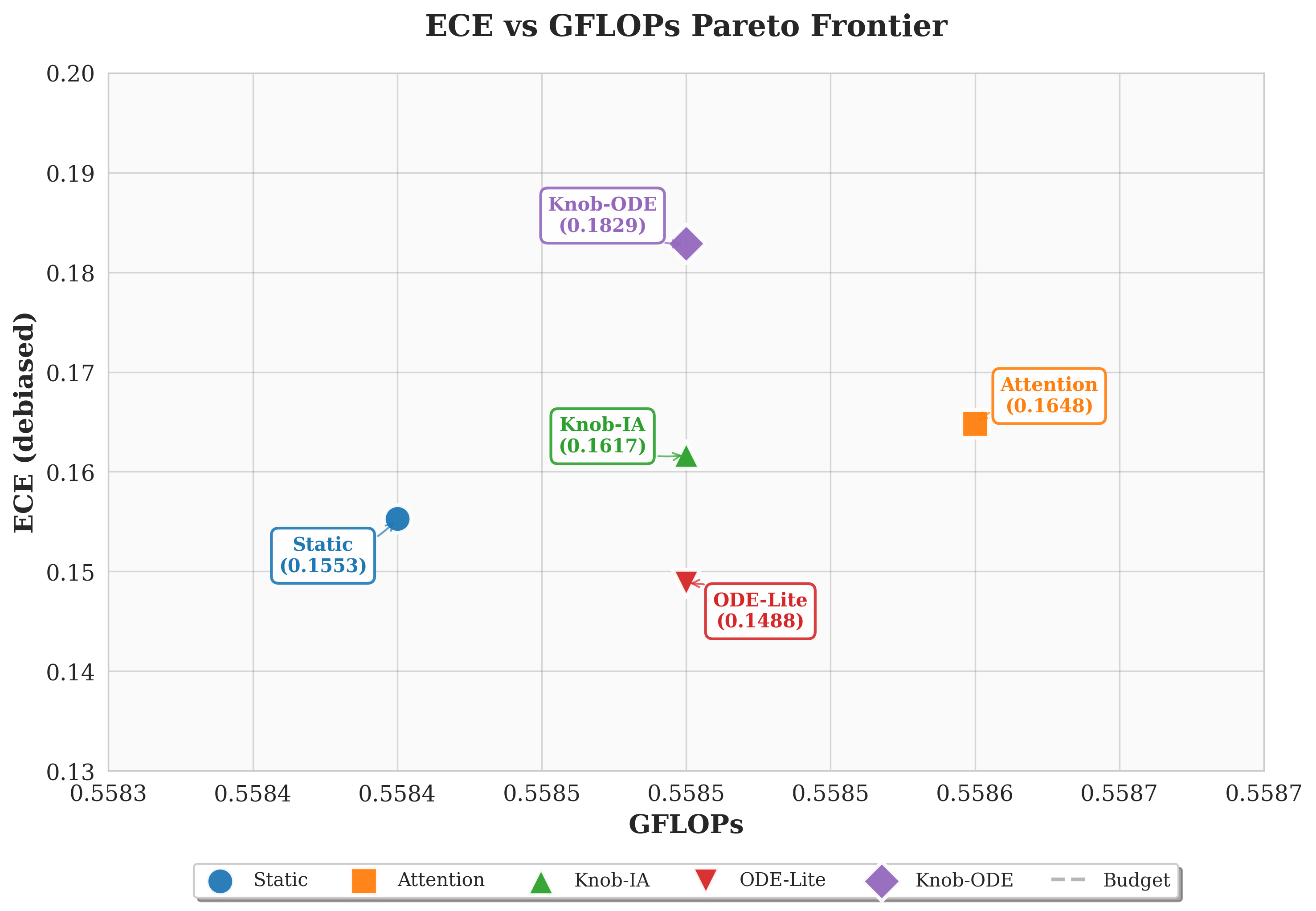}\hfill
  \includegraphics[width=0.48\linewidth]{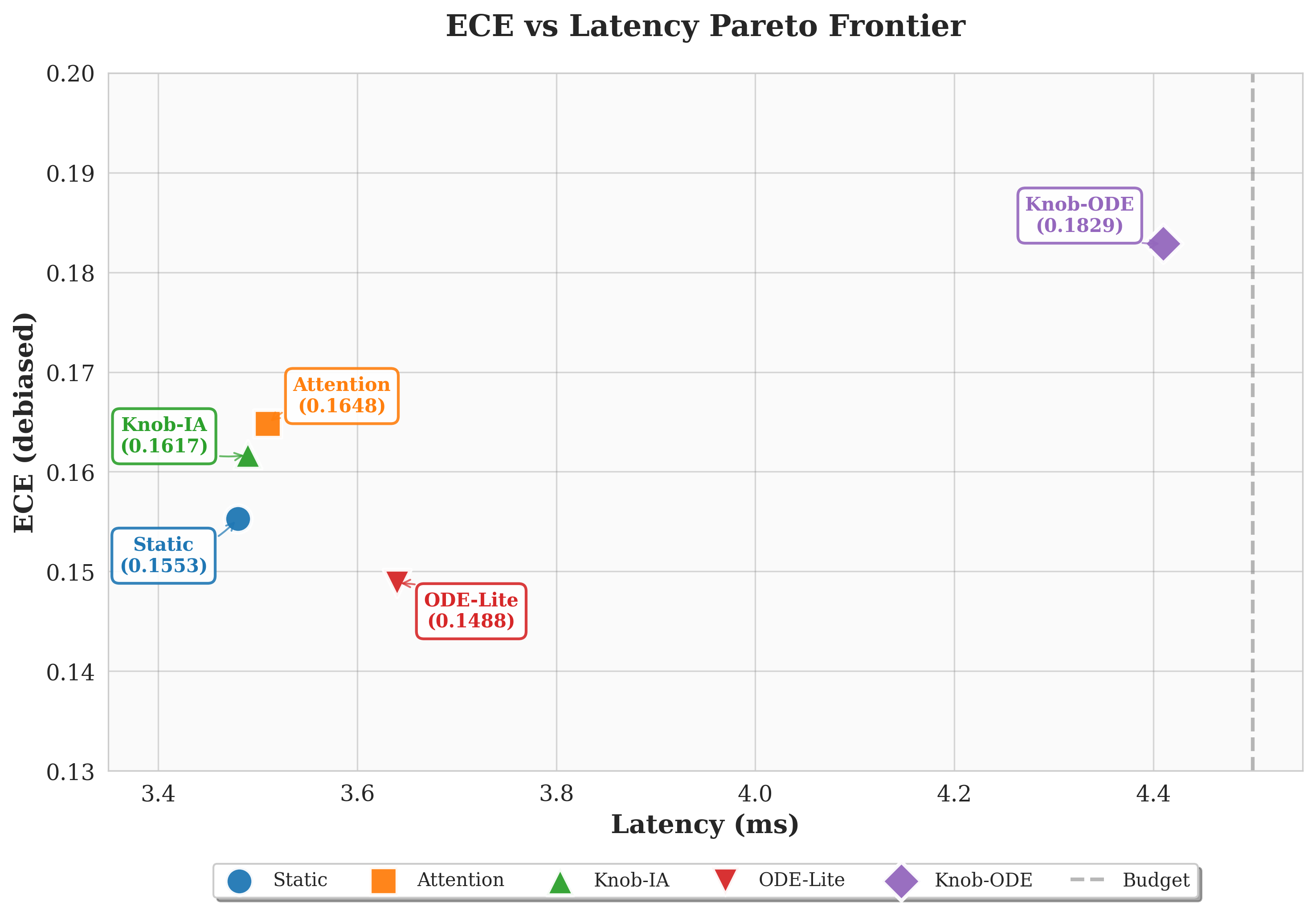}
  \caption{\textbf{Cost--calibration Pareto frontiers.} Left: ECE$_{\mathrm{deb}}$ vs.\ GFLOPs. Right: ECE$_{\mathrm{deb}}$ vs.\ latency. \textbf{ODE-Lite} lies near the frontier in both views, indicating a favorable calibration--cost trade-off.}
  \label{fig:pareto}
\end{figure}

\begin{table}[t]
\centering
\caption{\textbf{Efficiency comparison.} The gate controller introduces only $\mathcal{O}(D)$ incremental overhead, and overall computational costs are essentially identical across methods.}
\label{tab:efficiency}
\small
\begin{tabular}{lccc}
\toprule
Method & Params (M) & GFLOPs & Latency (ms) \\
\midrule
\textbf{Static}    & 11.70 & 0.5584 & 3.48 \\
\textbf{Attention} & 11.90 & 0.5586 & 3.51 \\
\textbf{Knob-IA}   & 11.77 & 0.5585 & 3.49 \\
\textbf{ODE-Lite}  & 11.77 & 0.5585 & 3.64 \\
\textbf{Knob-ODE}  & 11.77 & 0.5585 & 4.41 \\
\bottomrule
\end{tabular}
\end{table}

\subsection{Design and Evidence Chain: Probe Experiments}
\label{sec:exp-probes}

To validate the principles underlying our framework, we conduct a series of probe experiments, mapping our design choices to a clear evidence chain.

\subsubsection{E-1 (Static Layer): Convex Fusion as Input-Adaptive Temperature}
\label{sec:exp-e1}

\textbf{Take-home message:} Convex fusion tends to apply a sample-wise temperature $T_{\mathrm{eff}}(x) \ge 1$, which ensures the model becomes more conservative (CSR $<$ 1) on more difficult inputs.

Our core theoretical claim is that \textbf{logit-level convex fusion} acts as an \textbf{input-adaptive temperature scaling} mechanism. To verify this, we conduct static probe experiments on \textsc{CIFAR-10-C}, focusing on samples where the two logit branches disagree, as this is where confidence modulation is most critical.

The results strongly support our interpretation (Fig.~\ref{fig:e1}). First, the confidence--shrinkage ratio (CSR) remains below $1$ and decreases as corruption severity rises, showing that the fusion becomes more conservative on harder inputs. Second, the effective temperature $T_{\mathrm{eff}}(x)$, which we estimate as $\widehat{T}(x)$, is consistently $\ge 1$ and correlates positively with severity. These observations confirm that the framework dynamically applies higher ``temperature'' to more challenging samples to suppress overconfidence.

\begin{figure}[t]
  \centering
  \includegraphics[width=0.32\linewidth]{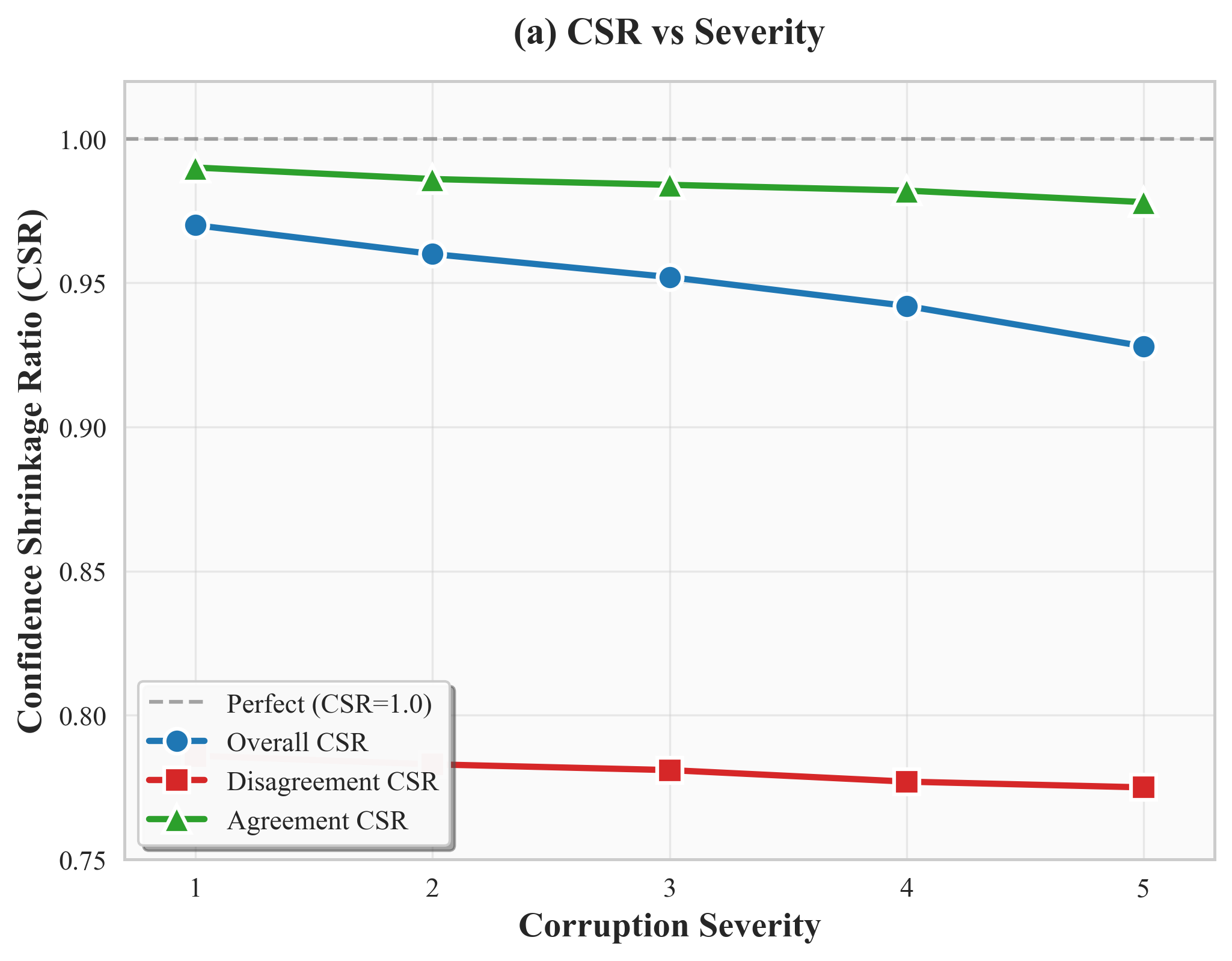}\hfill
  \includegraphics[width=0.32\linewidth]{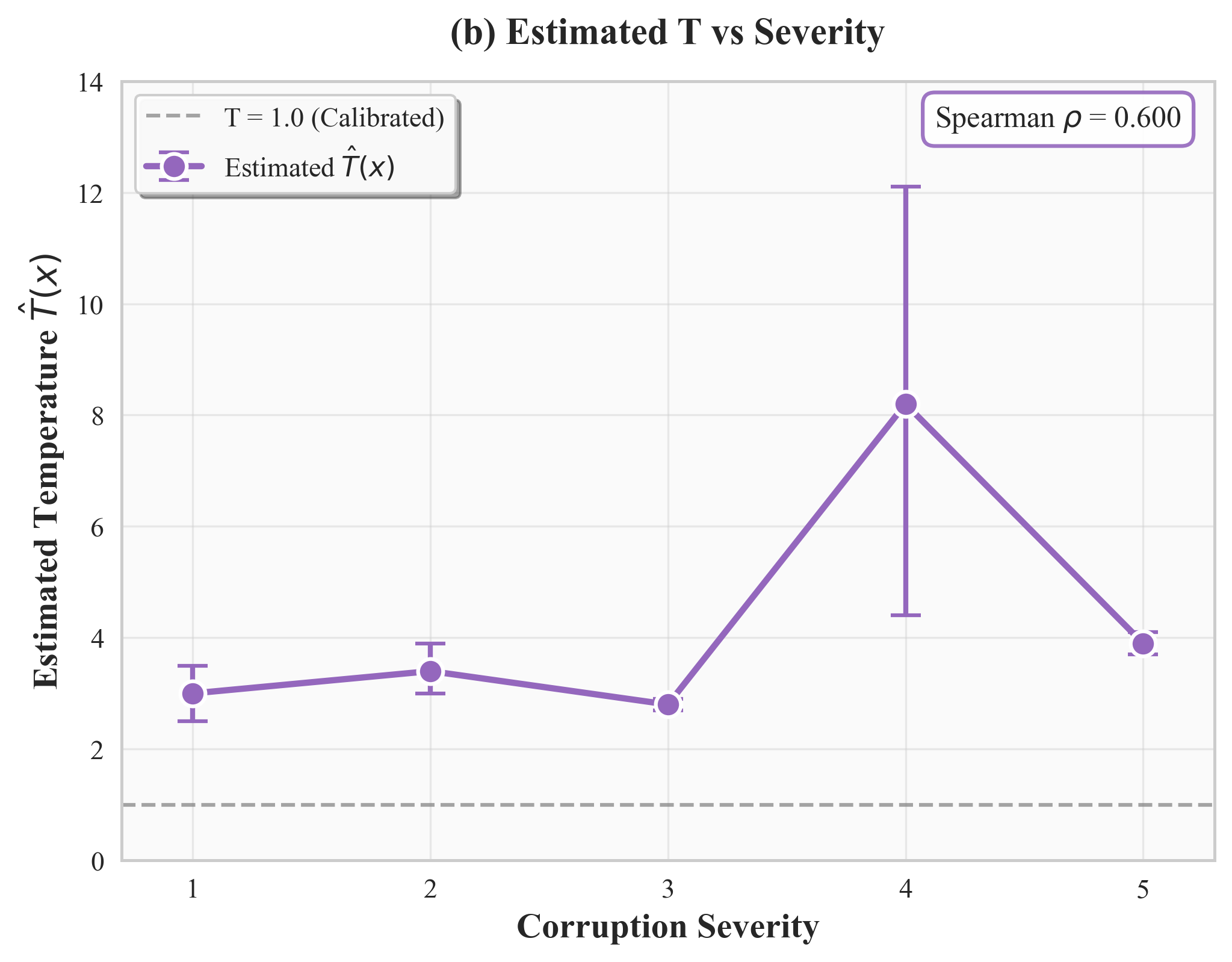}\hfill
  \includegraphics[width=0.34\linewidth]{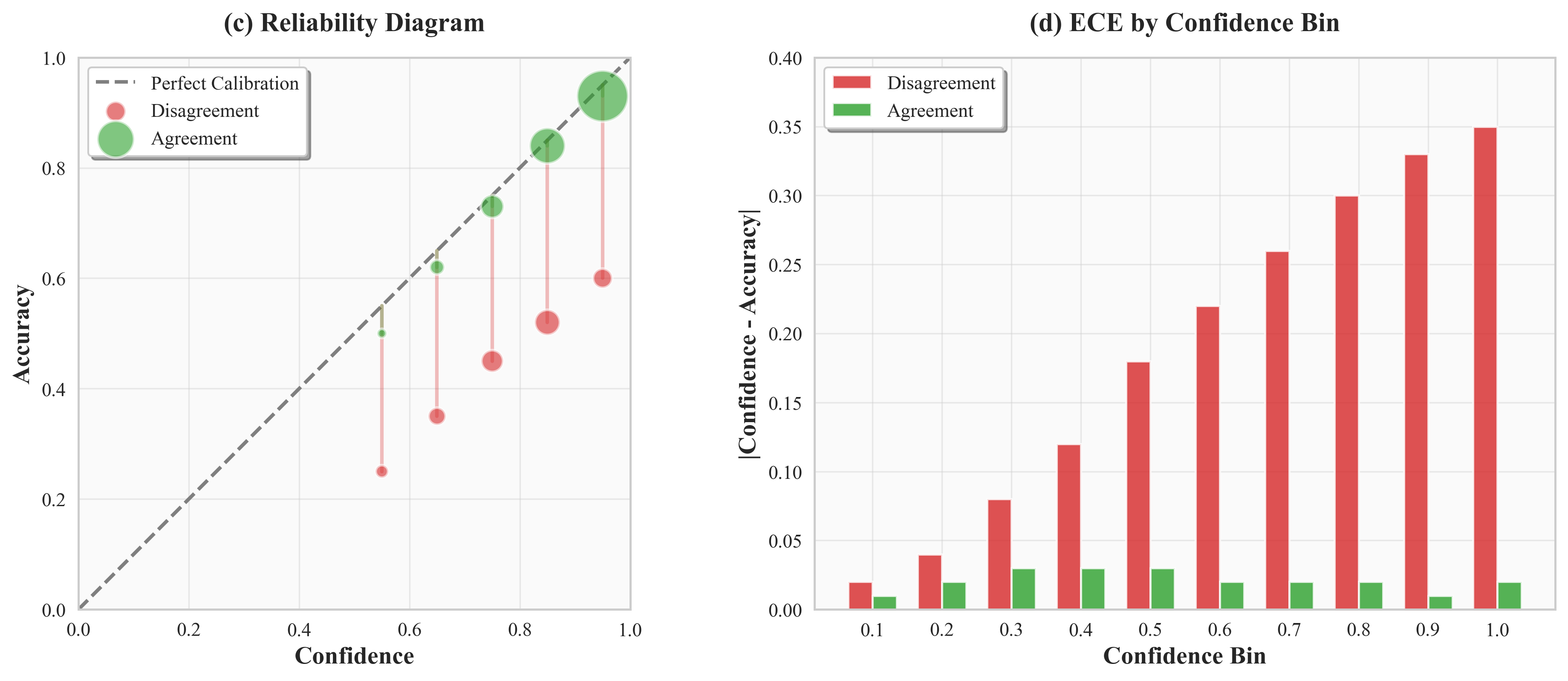}
  \caption{\textbf{E-1: Static equivalence evidence (Knob-IA).} (a) CSR for the overall, disagreement, and agreement subsets stays below $1$ and decreases with corruption severity. (b) $\widehat{T}(x)\!\ge\!1$ and increases with severity (Spearman's $\rho=0.6$). (c) Reliability diagram comparing disagreement and agreement subsets. (d) Per-bin ECE by confidence bins.}
  \label{fig:e1}
\end{figure}

\subsubsection{E-2 (Learning Layer): Gradient-Driven Gate ``Tilts Toward Advantage''}
\label{sec:exp-e2}

\textbf{Take-home message:} The gate learns to ``tilt toward the advantageous branch'' (AUC $>$ 0.5 and rising with training), while maintaining low polarization to avoid premature saturation.

Does the gate $g(x)$ learn a meaningful selection rule? To test this, we investigate whether the gate dynamically favors the logit branch that is more advantageous (i.e., yields lower loss) for a given sample. We define two metrics for the ``better branch'': Oracle Advantage (OA), based on final loss, and Gradient-Consistent Advantage (GCA), based on the instantaneous gradient direction.

As shown in Fig.~\ref{fig:e2-gate}, the gate's alignment with both advantage metrics (measured by ROC-AUC) is well above random chance (0.5) and improves throughout training. This confirms the gate learns an effective policy to ``pick sides.'' Crucially, the gate avoids premature saturation to binary 0/1 values, as quantified by a low polarization index ($\Pi$). This \textbf{high-discriminability, low-saturation} behavior is ideal: the gate reliably favors the better branch while maintaining soft, adaptive responses.

\begin{figure}[t]
  \centering
  \includegraphics[width=0.95\linewidth]{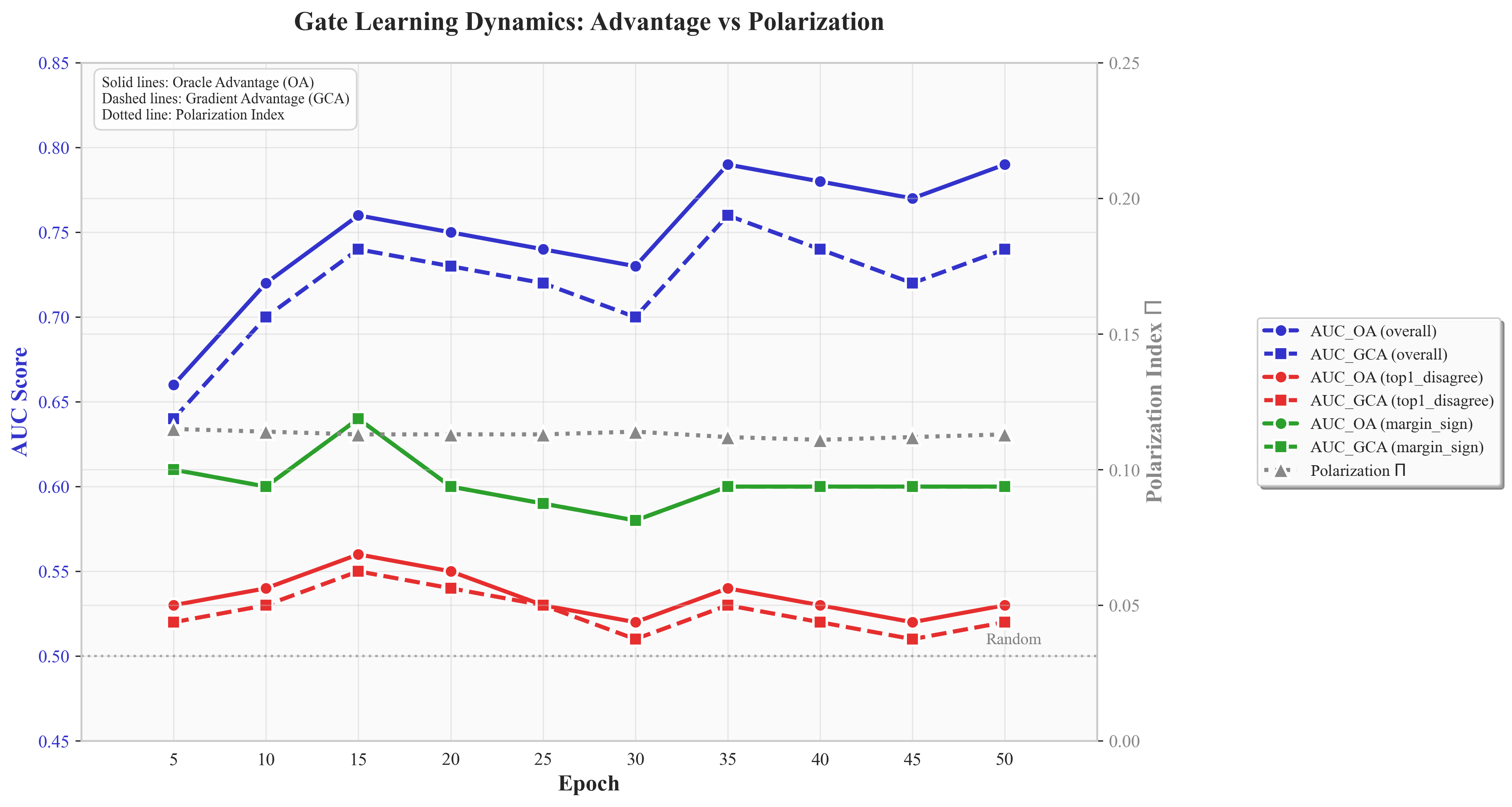}
  \caption{\textbf{E-2: Learning-level probe---gate ``taking sides.''} Treating $g(x)$ as a ``dynamic-preferred'' score, we plot AUC for Oracle Advantage (OA, solid lines) and Gradient-Consistent Advantage (GCA, dashed lines) across different sample subsets (overall, top-1 disagreement, margin sign). The polarization index $\Pi=\mathbb{E}[\,|g-0.5|\,]$ is shown on the right axis (gray triangles). AUC steadily increases during training, with overall $\mathrm{AUC}_{\mathrm{OA}} \approx 0.79$ and $\mathrm{AUC}_{\mathrm{GCA}} \approx 0.74$ by epoch 50, while $\Pi$ remains low ($\approx 0.10$), indicating the gate learns to discriminate without saturating.}
  \label{fig:e2-gate}
\end{figure}

\subsubsection{E-3 (Dynamic Layer): Verification of Control Dynamics}
\label{sec:exp-e3}

\textbf{Take-home message:} In Continuous Mode, the gate exhibits damped low-pass behavior on controlled data streams, confirming the physical prior is functional.

To validate the control-theoretic properties, we switch the model to \textbf{Continuous Mode} (state preservation enabled). Crucially, to isolate the gate's response to distribution shift rather than content changes, we construct a \textbf{Controlled Stream}: we fix a single base image (or a small set of fixed images) and apply a time-varying corruption severity schedule $s(t)$. This ensures that temporal variations in the gate $g(t)$ are driven primarily by the changing difficulty of the input, simulating a camera entering a foggy zone.

\paragraph{Step Response.} We simulate an abrupt ``step'' in distribution shift using a ``shot-noise'' step schedule from a low-severity regime ($s(t)=1$) to a high-severity regime ($s(t)=5$), and back to $s(t)=1$ (Fig.~\ref{fig:e3-combined}a). At each step we resample the corruption noise while keeping the underlying image content fixed, so that the temporal response is driven by shift intensity rather than content changes. As shown in Fig.~\ref{fig:e3-combined}a, the gate $g(t)$ does not jump instantaneously. Instead, it traces a smooth, damped trajectory. Note that even small adjustments in the gate value $g(t)$ (blue line) effectively modulate the output confidence $p_{\max}(t)$ (orange line), demonstrating the mechanism's sensitivity. This confirms that the model suppresses high-frequency jitter during transitions, a property unattainable by standard frame-by-frame inference.

\paragraph{Frequency Response.} Under a sinusoidal frequency sweep (Fig.~\ref{fig:e3-combined}b), the system exhibits characteristics of a low-pass filter. As the input perturbation frequency increases, we observe a reduction in the response magnitude (negative dB) relative to the steady state. While the empirical attenuation curve is shallower than an ideal continuous system due to the discrete sampling of corruption severities, the downward trend confirms that \textbf{Knob-ODE} \textit{tends to attenuate} high-frequency input jitter while remaining responsive to lower-frequency distributional shifts, acting as a soft low-pass filter.

\begin{figure}[t]
\centering
\includegraphics[width=0.95\linewidth]{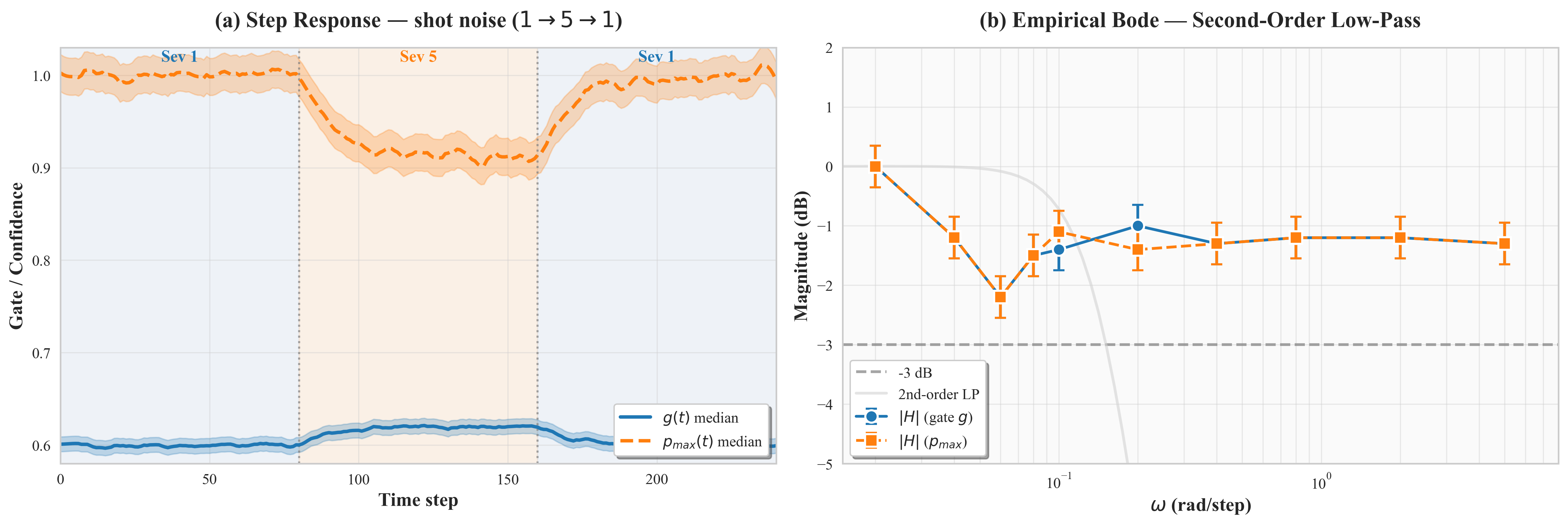}
\caption{\textbf{E-3: Dynamics-level probe (Continuous Mode).}
(a) \textbf{Step Response:} Under a shot-noise severity schedule ($1 \to 5 \to 1$), the gate $g(t)$ (blue) exhibits a smooth, damped transition rather than an instantaneous jump. Shaded bands indicate the inter-quartile range (IQR) over stochastic corruption realizations.
(b) \textbf{Empirical Bode Plot:} The magnitude response relative to DC (0\,dB) shows a downward trend as frequency increases, consistent with \textit{soft} low-pass filtering behavior. Error bars represent standard error of the mean (SEM).}
\label{fig:e3-combined}
\end{figure}

\paragraph{Summary of Evidence.}
The E-1, E-2, and E-3 probe experiments provide a cohesive evidence chain. They confirm that \emph{convex fusion} tends to implement \emph{input-adaptive temperature}, the gate learns to \emph{follow the loss gradient}, and the \textbf{Knob-ODE} dynamics provide robust \emph{low-pass stability} when operated in Continuous Mode.

\section{Limitations}
\label{sec:limitations}

We acknowledge several limitations of this work:
\begin{itemize}
    \item \textbf{Dataset scope:} Our experiments are limited to CIFAR-10/CIFAR-10-C. Validation on larger-scale datasets (e.g., ImageNet-C) and other modalities remains future work.
    \item \textbf{Dual-stream requirement:} The current framework requires a dual-stream backbone architecture, which may not be directly applicable to all existing models without modification.
    \item \textbf{Run-time tuning not demonstrated:} Although our framework structurally supports real-time adjustments of $\zeta$ and $\omega_n$, systematic user studies or detailed parameter sweeps have yet to be conducted to empirically validate this feature.
    \item \textbf{Continuous Mode dependency:} The full benefits of the second-order dynamics (Knob-ODE) are realized only in Continuous Mode with temporally correlated streams; for i.i.d. tasks, simpler variants (Knob-IA, ODE-Lite) may be preferable.
    \item \textbf{Scope boundary:} We intentionally prioritize interface clarity and control signatures over exhaustive benchmark coverage; our goal is to establish a design axis rather than a leaderboard entry.
\end{itemize}

\section{Conclusion}
\label{sec:conclusion}

We introduced Knob, an exploratory framework that embeds classical control theory into neural gating. By modeling the confidence gate through a physical system analogy, we achieve dual advantages: moderated confidence through convex fusion in static settings and smooth, damped transitions in dynamic inference scenarios. Our results on CIFAR-10-C suggest that explicitly constraining neural dynamics with physical laws is a viable path toward robust and interpretable inference. We hope this work inspires further research into physics-grounded neural architectures and human-in-the-loop AI systems.

\section{Broader Vision and Open Questions}
\label{sec:vision-open}

The primary contribution of this work is the \textbf{``Volume Knob''} interface. We envision a deployment scenario where $\zeta$ and $\omega_n$ are not fixed hyperparameters, but exposed control surfaces.

\paragraph{Run-time Human-in-the-Loop Control.}
This framework is architecturally designed to support run-time intervention. A clinician or safety engineer could adjust the ``Stability'' knob ($\zeta$) in real-time. Increasing $\zeta$ would theoretically force the model to be more conservative, requiring stronger inter-branch agreement to shift its state. This capability is structurally enabled by our ODE formulation, which decouples the control parameters from the feature extraction weights.

We emphasize that this paper primarily introduces the architectural interface and control mechanisms. Comprehensive user studies and extensive run-time control analyses remain as future directions.

\paragraph{Potential Applications.}
\begin{itemize}
    \item \textbf{Medical Diagnosis:} Clinicians adjust $\zeta$ for conservative predictions on ambiguous cases.
    \item \textbf{Autonomous Driving:} Increase $\omega_n$ for rapid response in dynamic environments, or decrease for smoother predictions in stable conditions.
    \item \textbf{Video Analytics:} In Continuous Mode, the framework naturally smooths predictions over time, reducing flicker in detection outputs.
\end{itemize}

\paragraph{Open Questions.}
Future work should explore: (1) Applying Knob to single-stream models via temporal self-ensembling; (2) Extending the physical metaphor to higher-order systems (e.g., PID controllers) for more complex regulation; (3) Conducting user studies to verify the utility of this interface in real-world human-AI teams; and (4) Scaling experiments to larger datasets such as ImageNet-C.

\bibliographystyle{plainnat}
\bibliography{references}

\appendix
\section{Additional Experimental Details}
\label{app:hyperparams}

\paragraph{Hyperparameters.}
All models are trained with the following settings:
\begin{itemize}
    \item Optimizer: AdamW with weight decay $1 \times 10^{-4}$
    \item Learning rate: $3 \times 10^{-4}$ with cosine annealing
    \item Batch size: 128
    \item Epochs: 100
    \item Mixed precision: FP16 via automatic mixed precision
\end{itemize}

\paragraph{Data Preprocessing.}
Standard CIFAR-10 preprocessing: random horizontal flip, random crop with padding=4, normalization with mean $(0.4914, 0.4822, 0.4465)$ and std $(0.2470, 0.2435, 0.2616)$.

\paragraph{Curriculum Sampling.}
We gradually increase corruption severity during training according to the schedule: epochs 1--30 use severity 1--2, epochs 31--60 use severity 1--3, epochs 61--100 use all severities 1--5. We use a curriculum over augmentation strength during training to stabilize gate learning; all compared methods use the same data pipeline. Our goal is an internal comparison of architectural interfaces rather than direct leaderboard comparability to works trained under different protocols.

\section{Theoretical Analysis of Convex Fusion}
\label{app:temp}

Here we provide additional analysis of why logit-level convex fusion tends to act as input-adaptive temperature scaling.

\begin{proposition}[Margin Contraction under Convex Fusion]
Let $\bm z_1, \bm z_2 \in \mathbb{R}^K$ be two logit vectors and $g \in [0,1]$ be a gate parameter. Let $m_1 = z_1^{(k^\star)} - z_1^{(j^\star)}$ and $m_2 = z_2^{(k^\star)} - z_2^{(j^\star)}$ denote the top-2 margins, where $k^\star$ is the agreed-upon top-1 class and $j^\star$ is the runner-up. If both branches agree on $k^\star$, the fused margin satisfies:
\begin{equation}
m_{\mathrm{fuse}} = g \cdot m_1 + (1-g) \cdot m_2 \leq \max(m_1, m_2)
\end{equation}
with equality only when $m_1 = m_2$ or $g \in \{0, 1\}$.
\end{proposition}

\begin{proof}
This follows directly from the properties of convex combinations. For any $g \in (0,1)$ and $m_1 \neq m_2$:
\begin{equation}
g \cdot m_1 + (1-g) \cdot m_2 < g \cdot \max(m_1, m_2) + (1-g) \cdot \max(m_1, m_2) = \max(m_1, m_2)
\end{equation}
\end{proof}

The confidence reduction induced by this margin contraction can be interpreted as applying an effective temperature $T_{\mathrm{eff}} := \max(m_1,m_2)/m_{\mathrm{fuse}} \geq 1$, where the temperature scales with the degree of disagreement between branches. This interpretation is validated empirically in our E-1 experiments (Fig.~\ref{fig:e1}).

\section{Convexity of Cross-Entropy in the Gate}
\label{app:gate-convex}

\begin{lemma}
Fix two logit vectors $\bm z_1,\bm z_2\in\mathbb{R}^K$ and label $y$.
Let $\bm z(g)=g\bm z_1+(1-g)\bm z_2$ with $g\in[0,1]$.
Then the cross-entropy $\ell(g) = -z_y(g) + \log\sum_{k}\exp(z_k(g))$ is convex in $g$.
\end{lemma}

\begin{proof}
The function $\bm z \mapsto -z_y + \log\sum_k \exp(z_k)$ is convex in $\bm z$ (log-sum-exp is convex and $-z_y$ is linear). Since $\bm z(g)$ is affine in $g$, $\ell(g)$ is convex in $g$ by composition of a convex function with an affine map.
\end{proof}

This convexity implies that, for fixed branch logits, optimizing the gate $g$ is a well-behaved sub-problem. Note that the full training remains non-convex because $\bm z_1$ and $\bm z_2$ are themselves outputs of neural networks.

\section{Proof of Tustin Discretization Stability (Proposition~\ref{prop:tustin-stable})}
\label{app:tustin-proof}

The Tustin (or bilinear) transform maps the left-half of the complex $s$-plane to the interior of the unit circle in the $z$-plane. A continuous-time linear system is stable if and only if all eigenvalues of its state matrix $A$ have negative real parts (i.e., lie in the open left-half plane).

For a second-order system with $\zeta>0$ and $\omega_n>0$, the characteristic equation is:
\begin{equation}
s^2 + 2\zeta\omega_n s + \omega_n^2 = 0
\end{equation}
which has roots:
\begin{equation}
s_{1,2} = -\zeta\omega_n \pm \omega_n\sqrt{\zeta^2 - 1}
\end{equation}

For $\zeta > 0$, these roots have strictly negative real parts, making matrix $A$ Hurwitz and the continuous system stable.

The Tustin transform, by its definition $s \leftarrow \frac{2}{\Delta t}\frac{z-1}{z+1}$, maps the entire open left-half $s$-plane to the open unit disk $|z|<1$. Consequently, if the continuous-time system is stable, the eigenvalues of the discretized matrix $A_d$ are guaranteed to lie inside the unit circle, ensuring the stability of the discrete-time system for any choice of time step $\Delta t > 0$.

For a formal proof, see standard texts on digital control, e.g., \cite{franklin2019feedback}.

\section{Expanded Results and Curves}
\label{app:expanded}

Full per-corruption accuracy tables, reliability diagrams, and additional ablation studies are provided in the supplementary materials included with this submission.

\end{document}